\documentclass[letterpaper]{article}
\usepackage[preprint]{aaai2027}
\usepackage[hyphens]{url}
\usepackage{graphicx}
\usepackage{natbib}
\usepackage{caption}
\usepackage{booktabs}
\usepackage{makecell}
\usepackage{multirow}
\usepackage{algorithmic}
\usepackage{algorithm}
\usepackage{amsmath}
\usepackage{amssymb}
\usepackage{mathtools}
\usepackage{amsthm}
\newtheorem{theorem}{Theorem}
\newtheorem{proposition}{Proposition}
\newtheorem{lemma}{Lemma}
\theoremstyle{definition}
\newtheorem{definition}{Definition}
\newcommand{\mycomment}[1]{\hfill{\footnotesize $\triangleright$ #1}}
\newcommand{\LINECOMMENT}[1]{\STATE {\footnotesize $\triangleright$ #1}}
\newcommand{\FUNCTION}[1]{\STATE \textbf{Function} #1}
\newcommand{\ENDFUNCTION}{}

\title{ResiSpec: Enhancing Multi-Candidate Speculative Sampling via Residual Distribution Shaping}
\author{
    Zhi-Kai Chen\textsuperscript{\rm 1,\rm 2},
    Jun-Jie Tao\textsuperscript{\rm 3},
    Wei-Xiang Mao\textsuperscript{\rm 3},
    De-Chuan Zhan\textsuperscript{\rm 1,\rm 2},
    Han-Jia Ye\textsuperscript{\rm 1,\rm 2}\corresponding
}

\affiliations{
    \textsuperscript{\rm 1}School of Artificial Intelligence,
    Nanjing University, China\\
    \textsuperscript{\rm 2}National Key Laboratory for Novel Software Technology,
    Nanjing University, China\\
    \textsuperscript{\rm 3}Nanjing University, China\\
}
\begin{document}
\maketitle



\begin{abstract}

The efficiency of Large Language Model (LLM) serving is fundamentally limited by the sequential nature of autoregressive decoding. Speculative Decoding (SD) mitigates this by using a lightweight draft model to speculate future tokens, which are then validated by the LLM in a single parallel forward pass. To further boost efficiency, multi-candidate schemes propose diverse candidate sets to increase the likelihood of token acceptance. However, we show that these schemes are bottlenecked by Residual Drift: a phenomenon where the rejection of initial candidates causes the residual target distribution to diverge from the draft model's predictions. This shift renders subsequent candidates ineffective and forces the system into expensive resampling. To resolve this, we propose ResiSpec, a framework that strategically reforms the proposal distribution during verification to anchor the residual target mass within the draft model's high-confidence regions. By mathematically re-aligning the verification process without compromising output exactness, ResiSpec prevents candidate obsolescence and achieves up to 1.92$\times$ speedup over state-of-the-art multi-candidate methods. Code is available at \url{https://github.com/Czzzk/Resispec}.


\end{abstract}

\section{Introduction}

The widespread adoption of Large Language Models (LLMs)~\cite{gpt4,llama2} has made efficient inference a critical priority for reducing both user latency and operational costs~\cite{efficiency_motivation}. However, standard auto-regressive (AR) decoding is fundamentally limited by its sequential nature, generating tokens one-by-one~\cite{specd_2}. This process creates a severe I/O bottleneck, as the system must repeatedly access the expanding Key-Value (KV) cache to compute attention for each new output~\cite{orca}. This sequential dependency prevents the system from fully leveraging the massive parallel processing power of modern GPUs, leaving compute cores significantly underutilized while waiting for memory traffic~\cite{flash_attention,lookahead,specud}. Consequently, the efficiency of LLM serving remains low, as most execution time is consumed by the overhead of accessing context rather than performing computation.

To overcome these sequential constraints, Speculative Decoding (SD)~\cite{specud,specd_2,specbench} introduces a ``speculate-then-verify'' paradigm. The core idea is to decouple the generation process: a lightweight draft model rapidly predicts a sequence of candidate tokens at a fraction of the cost, which the larger target model then validates in a single, parallel forward pass~\cite{staged,rest,adaptive_l_sd}. This approach directly addresses the I/O bottleneck by exploiting the Transformer's inherent ability to process multiple tokens concurrently with nearly the same memory latency as a single-token step~\cite{all_you_need,megatron_LM,efficiently_transformer}. By shifting from one-by-one generation to bulk verification, SD effectively increases the number of generated tokens per memory-access cycle. Consequently, performance is contingent upon acceptance length, where longer validated sequences translate into a substantial reduction in total inference steps.

To fully exploit this parallel capacity and maximize the ``tokens-per-load'' efficiency, recent advancements have transitioned from linear, single-sequence speculation to multi-candidate schemes, such as tree-based or batch-sampling verification~\cite{specinfer, medusa, eagle}. The core idea is to expand the search space by pre-sampling multiple candidate tokens for the same or branching positions. By presenting the target model with a diverse set of ``choices'' in a single forward pass, these schemes aim to increase the statistical probability that at least one high-quality continuation is accepted. This strategy effectively seeks to maximize the step size—the number of confirmed tokens—of each inference iteration.

While multi-candidate speculative decoding schemes—such as tree-based or batch verification—have significantly advanced LLM inference efficiency by expanding the search space beyond linear sequences, they encounter a critical scaling bottleneck. Theoretically, presenting more candidates should increase the probability of finding a high-quality match; however, in practice, the marginal utility of additional candidates diminishes rapidly as the batch size or tree width grows. We observe that as the verification chain lengthens, the acceptance probability of each subsequent candidate drops progressively, often leading to performance plateaus or even efficiency regressions. This raises a fundamental question: why do existing multi-candidate frameworks fail to fully leverage a larger search space to achieve proportional speedups?

We formalize this scalability limit as Residual Drift. This phenomenon stems from the mathematical necessity of maintaining ``exactness'' during verification: when the target model rejects a draft candidate, it must compensate by shifting its sampling distribution to a residual distribution. Crucially, this residual is naturally concentrated in the draft model’s ``blind spots''—semantic regions where the lightweight draft model has low predictive confidence or fails to capture complex dependencies. Because all candidates in a multi-candidate scheme are pre-sampled from the original draft distribution (before any rejections occur), they are inherently ill-suited for this shifted residual. Consequently, an early rejection often triggers a ``cascade of obsolescence'': the target model's updated criteria move so far away from the draft's knowledge base that the remaining pre-sampled candidates become statistically irrelevant, forcing expensive resampling and negating the parallelism benefits of the multi-candidate design.



To resolve this, we propose ResiSpec (Figure~\ref{method_overview}). Our key insight is that the mathematical path to exactness is not unique; the drift observed in prior work is a byproduct of sub-optimal verification criteria rather than an inevitable cost of zero-bias sampling. ResiSpec reformulates the verification criteria by integrating an auxiliary proxy of draft density, which steers the rejection residuals back toward the draft model’s high-confidence zones. By reshaping the transition logic, ResiSpec ensures that pre-sampled candidates remain viable throughout the verification chain without compromising the statistical integrity of the target model. 


\begin{itemize}
\item Characterization of Residual Drift: We formalize \textit{Residual Drift} as a fundamental bottleneck in multi-candidate speculative decoding, where sequential rejections shift the target distribution into the draft model's blind spots, invalidating parallel candidates.
\item The ResiSpec Framework: We propose ResiSpec, which reshapes verification criteria by introducing an auxiliary proxy to keep rejection residuals proximal to the draft model's high-density regions. This ensures the viability of pre-sampled candidates even after preceding rejections.
\item Theoretical Rigor and Speedup: We prove ResiSpec's mathematical exactness and demonstrate up to 1.92$\times$ speedup over multi-candidate methods while maintaining zero-bias sampling across benchmarks.
\end{itemize}

\section{Related Work}
\textbf{Auto-regressive Acceleration and Speculative Decoding.} 
Large Language Models (LLMs) generate text auto-regressively ~\cite{gpt2,gpt3,gpt4,llama}, a process constrained by memory bandwidth. Speculative Decoding (SD)~\cite{specud, specd_2} emerged as a pivotal lossless acceleration paradigm, using a small draft model to generate candidate tokens that the target model verifies in parallel. While SD guarantees mathematical equivalence to the original distribution, its performance is capped by the draft model's hit rate on a single linear sequence~\cite{specbench,survey_2025,2024_nips_theoretical}.

\textbf{Multi-Candidate and Tree-Based Speculation.}
To overcome the limited acceptance rate of linear speculative decoding, recent work explores multiple candidate paths in parallel. \textit{SpecInfer}~\cite{specinfer} introduced tree-based attention, enabling the target model to verify branching token sequences in a single forward pass. Subsequent methods improve this paradigm along two directions: candidate generation and tree construction. For candidate generation, draft-model-free methods such as \textit{EAGLE}~\cite{eagle, eagle_2}, \textit{Medusa}~\cite{medusa}, \textit{Hydra}~\cite{hydra}, and \textit{KOALA}~\cite{koala} use specialized heads or lightweight modules to generate candidates from the target model's hidden states. For tree construction, \textit{Sequoia}~\cite{sequoia} uses dynamic programming to optimize tree shape under a budget, while prior work~\cite{optimal_2025} studies theoretical criteria for selecting tokens in the verification tree.


\textbf{Scalability Limits of Standard Verification Protocols}
Despite the success of multi-candidate speculation, standard sampling protocols fail to scale efficiently with candidate counts, leading to a bottleneck in acceptance length. We bridge this gap with ResiSpec, a new sampling paradigm tailored for multi-candidate verification. By optimizing the sampling logic for large-scale branching trees, ResiSpec significantly enhances the marginal utility of additional candidates and achieves a higher expected token yield.

\section{Preliminary}
\noindent\textbf{Speculative Decoding.} The key idea behind speculative decoding is to accelerate autoregressive sampling by leveraging a draft model to propose candidate tokens, which are then verified by the target model. Let the target model define a conditional distribution $P_{\texttt{tgt}}(x \mid x_{<t})$ and the draft model define $P_{\texttt{dft}}(x \mid x_{<t})$, where t is the current step. The algorithm operates in two stages: drafting and verification.

During the drafting stage, the draft model generates $\gamma \in \mathbb{Z}^+$ candidate tokens from its own sampling distribution $P_{\texttt{dft}}(x_t \mid x_{<t})$. Since these tokens are only approximate, a verification step is performed. In this stage, the target model evaluates the speculative tokens $t+1 \colon t+n$ in parallel using a single forward pass. Each token is accepted with probability
$\min(1, p_{\texttt{tgt}}(x) / p_{\texttt{dft}}(x))$.
If none of the speculative tokens are accepted, the target model resamples according to the adjusted distribution
\begin{equation}
p_{\texttt{tgt}}'(x) = \mathrm{norm}\Big(\max(0,\, p_{\texttt{tgt}}(x) - p_{\texttt{dft}}(x))\Big)
\label{eq:residual} 
\end{equation}
ensuring correctness while still leveraging the draft model for potential speedup.

\noindent\textbf{Multi-candidate Tree-structured Verification.} To explore multiple potential future paths, recent schemes organize $N$ speculative tokens into a tree-structured proposal $\mathcal{T}$. The tree $\mathcal{T}$ is characterized by its depth $d$, which specifies the number of tokens predicted sequentially forward along each path, and its branching factor $b$, which denotes the number of parallel candidates proposed at each position.

Specifically, at any given node in the tree, the draft model proposes $b$ independent tokens $\{x_1, \dots, x_b\}$ for the same speculative position. To preserve the exactness of the target distribution, these $b$ candidates are verified sequentially. Let $p_{\texttt{tgt}}^{(1)}(x)$ be the initial target distribution for the first candidate $x_1$. For each $i \in \{1, \dots, b\}$, candidate $x_i$ is accepted with probability: $\min(1, p_{\texttt{tgt}}^{(i)}(x_i) / q_{\texttt{dft}}(x_i))$
where $q_{\texttt{dft}}$ is the draft distribution. If $x_i$ is rejected, the target distribution for the next candidate $x_{i+1}$ is updated to the residual distribution $p_{\texttt{tgt}}^{(i+1)}$:
\begin{equation}
p_{\texttt{tgt}}^{(i+1)}(x) = \frac{\max\left(0, p_{\texttt{tgt}}^{(i)}(x) - q_{\texttt{dft}}(x)\right)}{\sum_{x' \in \mathcal{X}} \max\left(0, p_{\texttt{tgt}}^{(i)}(x') - q_{\texttt{dft}}(x')\right)}.
\label{eq:recur}
\end{equation}
In a tree structure, this verification process is applied at each level. If a candidate $x_i$ is accepted, the process continues to its children in the next level (depth); if all $b$ candidates at a certain position are rejected, the branch is truncated, and a final token is resampled from the cumulative residual distribution $p_{\texttt{tgt}}^{(b+1)}(x)$. By using a tree-attention mask, the target model can verify all paths within $\mathcal{T}$ in a single forward pass, significantly increasing the potential number of accepted tokens per step.

\section{Limitation of Multi-candidate Speculative Decoding}
Multi-candidate speculative decoding aims to maximize accepted tokens by verifying a batch of $N$ candidates $\{\mathbf{x}_1, \dots, \mathbf{x}_N\}$ for one position. Although a broader search space should improve acceptance, its efficiency is throttled by Residual Drift. In standard verification (Equation~\ref{eq:recur}), candidates are checked sequentially to preserve exactness. Once $\mathbf{x}_1$ is rejected, $\mathbf{x}_2$ must be verified against the residual distribution $P_{\text{res}}(x) = \max(0, p(x) - q(x)) / \sum_{x'} \max(0, p(x') - q(x'))$. By construction, $P_{\text{res}}(x)$ assigns zero probability where $q(x)$ already meets or exceeds $p(x)$, concentrating mass in the draft model's ``blind spots.'' This creates a severe support mismatch: all $N$ candidates are sampled i.i.d. from the original draft distribution $q(x)$, while later verification criteria shift away from $q(x)$'s high-density regions.

\begin{figure}[t]
\centering
\centerline{\includegraphics[width=\linewidth]{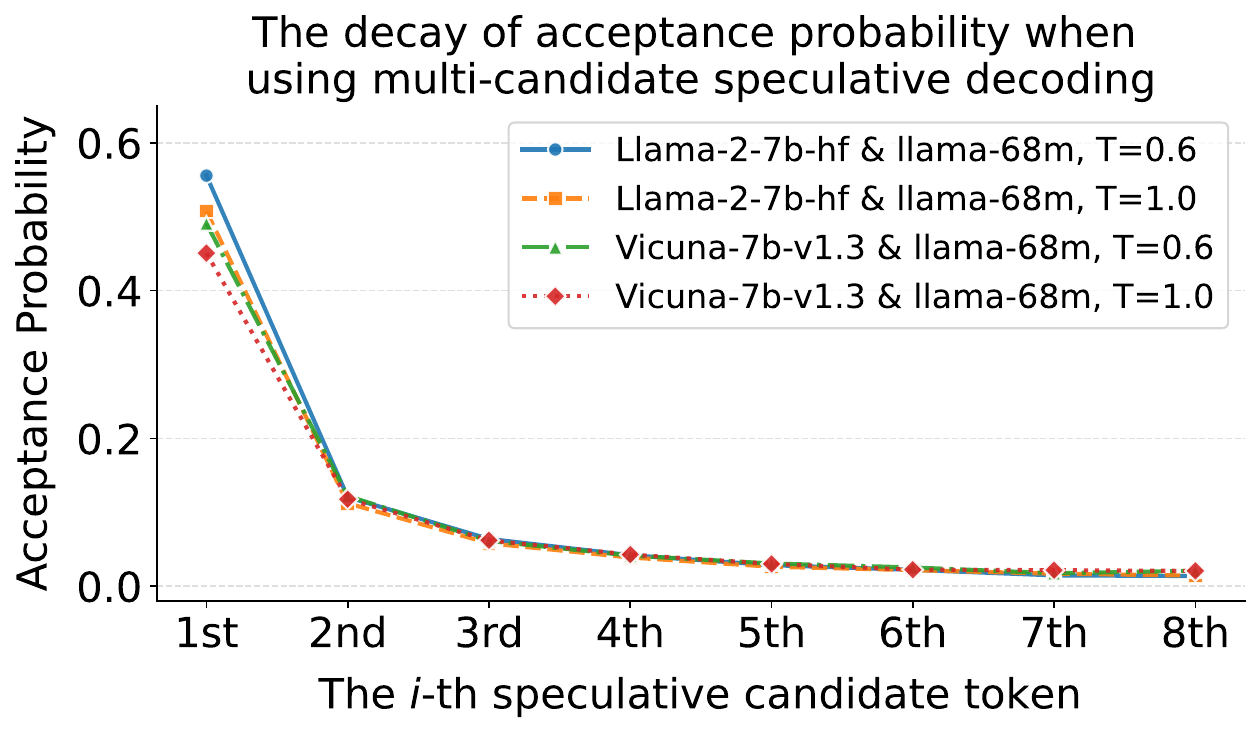}}
\caption{
\textbf{The Decay of Acceptance Probability.} Sequential candidate acceptance rates on CNN/DM using Llama-2-7B and Vicuna-7B~\cite{vicuna} as targets with Llama-68M as the draft. Across temperatures ($T \in \{0.6, 1.0\}$), acceptance drops sharply after the first candidate, leaving later candidates with near-zero marginal utility.
}
\label{candi_accept_rob}
\end{figure}

Our empirical results confirm this precipitous decline in marginal utility. As illustrated in Figure~\ref{candi_accept_rob}, the first candidate $\mathbf{x}_1$ maintains a meaningful acceptance rate, while the probability for subsequent candidates $\mathbf{x}_{i \geq 2}$ rapidly approaches zero. This pattern indicates a cascade of obsolescence: because candidates in the same batch are sampled from the high-density support of $q(x)$, an early rejection shifts the verification criteria toward a residual distribution where the remaining candidates have little statistical visibility. To quantify the resulting scaling inefficiency, we vary the number of candidates $N$ under a fixed speculation depth and report the efficiency score $\eta$, defined as the marginal EAL gain per additional candidate. Appendix Table~\ref{tab:scaling_inefficiency} shows that increasing $N$ yields only modest EAL improvements while $\eta$ drops sharply, indicating that wider speculation spends increasing parallel verification compute on candidates with diminishing acceptance utility.

These observations reveal a scaling wall: expanding tree width incurs linear verification cost but yields sub-linear acceptance gains because the residual distribution diverges from the draft distribution. This raises a central question: \textit{Can we break this bottleneck by mitigating Residual Drift?} If rejection residuals can remain aligned with high-density regions of $q(x)$, subsequent candidates may stay useful even after earlier failures. By re-aligning the verification criteria with the draft model's support, we can transform multi-candidate decoding from a game of diminishing returns into a scalable acceleration paradigm. Next, we present ResiSpec.




\begin{figure*}[ht]
\centering
\centerline{\includegraphics[width=\linewidth]{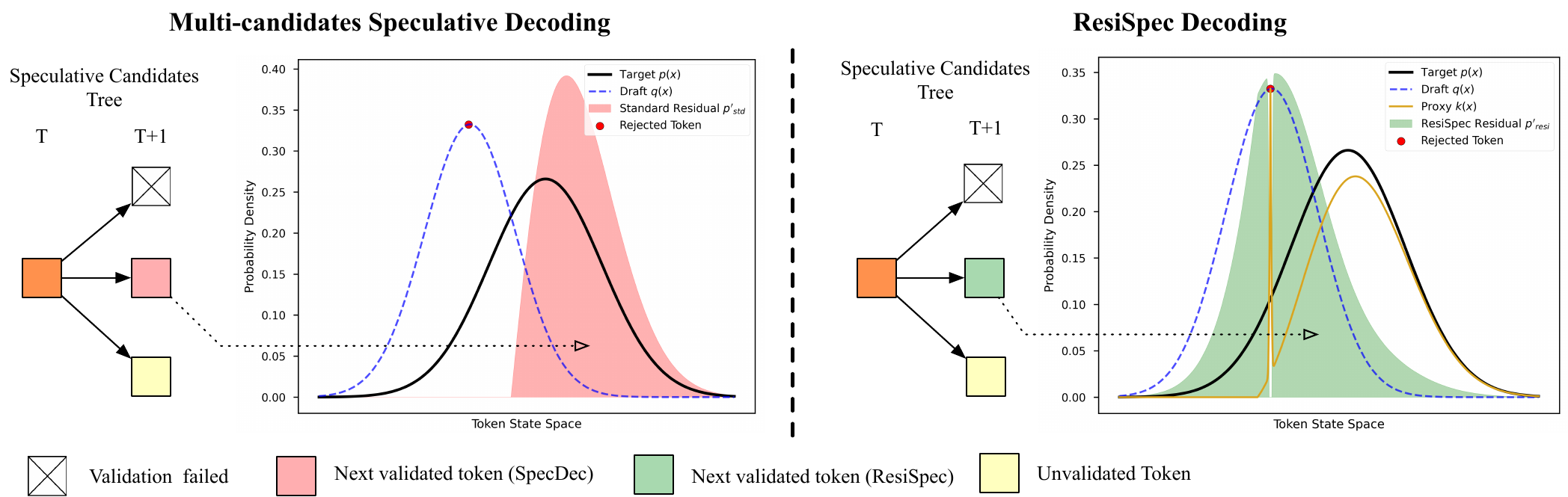}}
\caption{
Comparison of ResiSpec vs. standard multi-candidate verification. In traditional schemes, sequential rejections trigger Residual Drift, shifting the target distribution into regions mathematically incompatible with the draft model (red areas), causing candidates to become obsolete. ResiSpec mitigates this drift by employing an auxiliary proxy to reshape the residual distribution (green areas), re-aligning it with the draft model’s high-density support. This proximity ensures that pre-sampled candidates remain viable, effectively sustaining high acceptance rates across the entire verification batch.}
\label{method_overview}
\end{figure*}

\section{Method: Residual
Distribution Shaping}

In this section, we present ResiSpec, a framework that ensures high acceptance rates in multi-candidate speculative decoding by actively shaping the residual distribution. The innovation is the introduction of an auxiliary proxy distribution $k(x)$, which replaces the original draft distribution $q(x)$ during the verification step. By strategically constructing $k(x)$, we ensure that if a candidate is rejected, the resulting residual distribution $p'(x)$ remains proximal to the draft distribution $q(x)$, thereby maximizing the acceptance probability for all subsequent candidates in the batch.

%


\subsection{Problem Formulation}
In standard speculative decoding, a candidate $x$ is accepted with probability $\min(1, p(x)/q(x))$. If rejected, the algorithm samples from a residual distribution. To gain control over the failure state, we propose verifying candidates using a proxy distribution $k(x)$ in place of $q(x)$.

\begin{definition}[Proxy Mechanism]
\label{def:proxy_mech}
Given a target distribution $p(x)$ and an unnormalized residual mass $r(x)$, the proxy distribution $k(x)$ is defined as:
\begin{equation}
    k(x) = p(x) - r(x) + s(x)
\end{equation}
where $r(x)$ represents the mass reserved for the residual distribution $p'(x) = r(x)/Z$, and $s(x) \geq 0$ is a slack function. 
\end{definition}

The slack function $s(x)$ serves as a compensatory mass to restore normalization of $k(x)$. To prevent this injected mass from interfering with the reserved distribution, $s(x)$ is only permitted to be non-zero where the residual mass is absent ($r(x)=0$). This leads to the disjoint support constraint:
\begin{equation}
    \text{supp}(s) \cap \text{supp}(r) = \emptyset, \quad \text{i.e., } s(x) \cdot r(x) = 0, \forall x \in \mathcal{X}
\end{equation}

%

For the sampling process to remain exact (i.e., the marginal distribution is $p(x)$), we impose the constraints on $k(x)$:
\begin{enumerate}
    \item \textbf{Probability Constraint}: $\sum_x k(x) = 1$ and $k(x) \geq 0$.
    \item \textbf{Local Alignment}: At the sampled point $x_i$, $k(x_i) = q(x_i)$, ensuring consistency with the draft model's local probability.
\end{enumerate}

\begin{proposition}[Global Mass Identity]
\label{prop:mass_identity}
For the proxy $k(x)$ to be a valid probability distribution, the total mass of the slack function $s(x)$ must be exactly equal to the total reserved residual mass $Z = \sum_x r(x)$, provided $r(x) \leq p(x)$ for all $x \in \mathcal{X}$.
\end{proposition}

\begin{proof}
To satisfy the normalization axiom $\sum_x k(x) = 1$, we expand the definition:

\begin{equation}
\begin{aligned}
    \sum_x k(x) &= \sum_x (p(x) - r(x) + s(x)) = 1 \\
    &\Rightarrow \sum_x p(x) - \sum_x r(x) + \sum_x s(x) = 1 \\
    &\Rightarrow 1 - Z + \sum_x s(x) = 1 \Rightarrow \sum_x s(x) = Z.
\end{aligned}
\end{equation}
This identity reveals a fundamental conservation law in our framework: any probability mass ``reserved'' from the target to form the residual must be precisely balanced by the slack mass ``injected'' to construct the proxy.
\end{proof}

\begin{definition}[The Construction Objective]
\label{def:l1_obj}
Since subsequent candidates are pre-sampled from $q(x)$, the probability of future acceptance is maximized when the overlap between the residual $p'(x)$ and the draft $q(x)$ is maximized. This is equivalent to minimizing the $L_1$ distance:
\begin{equation}
\begin{split}
\min_{k} \sum_x |p'(x) - q(x)|, \quad \text{where } p'(x) = \frac{r(x)}{Z},\\
Z = \sum_x r(x).
\end{split}
\end{equation}
\end{definition}

\subsection{Theoretical Boundaries: Ideal Shaping and Limits}
To derive the optimal residual shaping strategy, we seek a formulation that satisfies the rigorous exactness constraints established in Definition~\ref{def:proxy_mech} while attaining the global minimum of the $L_1$ objective defined in Definition~\ref{def:l1_obj}.

\textbf{Ideal Alignment and the Capacity Bound.}
We begin by identifying the ideal shape of the residual mass $r(x)$ and then determine the maximum budget $Z$ that the system can sustain without violating probability axioms.

\begin{lemma}[Global Optimality of Residual Matching]
\label{lem:l1_matching}
To achieve the global minimum of the $L_1$ objective $\mathcal{J} = \sum_{x \in \mathcal{X}} |r(x)/Z - q(x)|$, the residual mass must satisfy the ideal scaling relationship:
\begin{equation}
\label{equ:ideal_r}
    r(x) = Z \cdot q(x), \quad \forall x \in \mathcal{X}
\end{equation}
\end{lemma}

\begin{proof}
By the properties of the $L_1$ norm, $\mathcal{J} \geq 0$, with equality if and only if the normalized residual $p'(x) = r(x)/Z$ is identical to $q(x)$. This directly implies $r(x) = Z q(x)$.
\end{proof}

\begin{proposition}[The Capacity Bound of Ideal Shaping]
\label{prop:validity}
For the ideal residual $r(x) = Z q(x)$ to be part of a valid proxy distribution $k(x) = p(x) - r(x) + s(x)$ while adhering to the disjoint support constraint ($s \cdot r = 0$), the total reserved mass $Z$ is bounded by the minimum density ratio:
\begin{equation}
    Z \leq \min_{x: q(x)>0} \frac{p(x)}{q(x)}
\end{equation}
\end{proposition}

\begin{proof}
Substituting the ideal shape $r(x) = Z q(x)$ into the non-negativity constraint $k(x) \geq 0$ from Definition~\ref{def:proxy_mech}, we have:
\begin{equation}
    p(x) - Z q(x) + s(x) \geq 0
\end{equation}
For any $x$ where $q(x) > 0$, the ideal residual $r(x) = Z q(x)$ is strictly positive. According to the disjoint support constraint in Definition~\ref{def:proxy_mech}, we must have $s(x) = 0$ for these tokens. The inequality thus simplifies to:
\begin{equation}
    p(x) - Z q(x) \geq 0 \implies Z \leq \frac{p(x)}{q(x)}
\end{equation}
Taking the minimum over all $x$ in the support of $q$ ensures the proxy remains non-negative, yielding the stated bound.
\end{proof}

\textbf{The Degeneracy and Support Paradox.} 
In real-world LLM inference, the theoretical feasible region defined by $Z \leq \min_{x: q(x)>0} \frac{p(x)}{q(x)}$ is often degenerate or operationally unsustainable. This stems from two fundamental conflicts:

\textit{1. The Vanishing Margin}: Due to the ``long-tail''~\cite{long_tail} nature of language models, there frequently exist tokens where the target model $p(x) \to 0$ while the draft model $q(x) > 0$. In such cases, the upper bound $\min \frac{p(x)}{q(x)}$ collapses to near-zero, effectively shrinking the feasible region to $\{0\}$. A budget of $Z \approx 0$ would deplete the residual mass $r(x)$, stripping the system of the probability budget needed to support subsequent candidates.

\textit{2. The Alignment-Identity Contradiction}: To achieve the global minimum of the $L_1$ objective ($L_1=0$), we require $r(x) = Z \cdot q(x)$. If the draft model has full support ($q(x) > 0, \forall x$), then $r(x) > 0$ must also hold. Under the \textit{disjoint support constraint} ($s \cdot r = 0$), this forces the slack function $s(x)$ to be zero everywhere, which, by Proposition~\ref{prop:mass_identity}, implies $Z = \sum s(x) = 0$. 

Crucially, accepting the degenerate solution $Z=0$ causes the proxy distribution to revert to the target ($k(x) = p(x)$). This directly violates the Local Alignment constraint $k(x_i) = q(x_i)$, as the rejected candidate $x_i$ is typically a point of over-estimation where $q(x_i) > p(x_i)$. We therefore reformulate the optimal shaping of the residual as a constrained approximation problem, relaxing the ideal shape to maintain a functional residual budget.

\begin{table*}[htbp]
\centering
\footnotesize
\begin{tabular}{@{}llccrrrrrr@{}}
\toprule
\multirow{2}{*}{Target LLM} &
\multirow{2}{*}{Draft Model} &
\multirow{2}{*}{T} &
\multirow{2}{*}{Dataset} &
\multicolumn{2}{c}{ResiSpec} &
\multicolumn{2}{c}{SpecInfer} &
\multicolumn{2}{c}{ResiSpec vs SpecInfer} \\
\cmidrule(lr){5-6} \cmidrule(lr){7-8} \cmidrule(lr){9-10}
& & & &
Acc.↑& Thru.↑&
Acc.↑& Thru.↑&
Acc. Gain & Speedup \\
\midrule
Llama-2-7B & JF68M & 1 & CNN/DM & 4.68 & 96.43 & 2.51 & 57.29 & 1.86$\times$& 1.68$\times$\\
Llama-2-7B & JF68M & 1 & OpenWebText & 4.69 & 96.06 & 2.51 & 56.76 & 1.86$\times$& 1.69$\times$\\
Llama-2-7B & JF68M & 1 & C4 & 4.67 & 96.90 & 2.64 & 61.02 & 1.77$\times$& 1.59$\times$\\
Llama-2-7B & JF68M & 0.6 & CNN/DM & 4.14 & 85.03 & 2.59 & 59.41 & 1.60$\times$& 1.43$\times$\\
Llama-2-7B & JF68M & 0.6 & OpenWebText & 4.16 & 85.62 & 2.49 & 56.48 & 1.67$\times$& 1.52$\times$\\
Llama-2-7B & JF68M & 0.6 & C4 & 4.13 & 85.91 & 2.66 & 61.69 & 1.55$\times$& 1.39$\times$\\
\bottomrule
\end{tabular}
\caption{Performance comparison between ResiSpec and the baseline SpecInfer. Regarding the configuration of the tree, we
use a 3-branch complete tree with a depth of 5. The results in the table are averaged over 200 runs to reduce variance and ensure statistical reliability. Experimental results demonstrate that ResiSpec significantly increased the accepted length (by up to 1.86$\times$) and throughput (by up to 1.68$\times$), validating its effectiveness in accelerating large language model inference.}
\label{tab:results}
\end{table*}

\begin{table}[htbp]
\centering
\small
\begin{tabular}{@{}lcccc@{}}
\toprule
\textbf{Method} & \makecell{\textbf{Acc.}\\\textbf{Len↑}} & \makecell{\textbf{Thru.}\\\textbf{(tok/s)↑}} & \makecell{\textbf{Acc. Len}\\\textbf{Gain}} & \textbf{Speedup} \\
\midrule
Sequoia & 3.07 & 88.14 & -- & -- \\
Sequoia+ResiSpec & 4.32 & 115.32 & 1.41$\times$ & 1.31$\times$ \\
EAGLE & 3.12 & 61.82 & -- & -- \\
EAGLE+ResiSpec & 3.85 & 71.17 & 1.23$\times$ & 1.15$\times$ \\
\bottomrule
\end{tabular}
\caption{Compatibility of ResiSpec with existing multi-candidate speculative decoding methods at $T=1.0$. ResiSpec is evaluated as an add-on to Sequoia~\cite{sequoia} and EAGLE~\cite{eagle} under matched decoding settings, improving both metrics.}
\label{tab:framework_compatibility}
\end{table}

\subsection{Practical Implementation: The ResiSpec Algorithm}


Given that the theoretical ideal is unattainable, we propose a minimal-budget approximation strategy. According to Theorem~\ref{thm:residual_shaping_monotonicity}, the $L_1$ distance between the shaped residual $p'(x)$ and the draft distribution $q(x)$ is strictly monotonically increasing with respect to $Z$ whenever $Z > Z^*$. Therefore, to maintain the fidelity of the draft model's trajectory and maximize the acceptance utility for subsequent candidates, we should seek the smallest possible $Z$ that satisfies the physical constraints of the verification step.



Since $x_0$ is a point where the draft model over-estimates the target ($q(x_0) > p(x_0)$), and no residual mass can be reserved at a point where the target density is already exhausted (i.e., $r(x_0) = 0$), the definition of $k(x)$ at $x_0$ simplifies to $q(x_0) = p(x_0) + s(x_0)$. This necessitates a local mass injection of:
\begin{equation}
    s(x_0) = q(x_0) - p(x_0)
\end{equation}
Given the Global Mass Identity $Z = \sum_x s(x)$ and the non-negativity requirement $s(x) \geq 0$, the total budget $Z$ is strictly lower-bounded by this local discrepancy:
\begin{equation}
    Z = s(x_0) + \sum_{x \neq x_0} s(x) \geq |q(x_0) - p(x_0)|
\end{equation}
This bound represents the minimum ``price'' the system must pay in terms of residual mass to bridge the gap between the draft's prediction and the target's reality at the failure point.

\begin{algorithm}[H]
  \caption{ResiSpec Sampling and Verification \\ 
  \label{alg:resispec}
  \textnormal{\small (\textit{The \underline{underlined step} marks ResiSpec's proxy computation, distinguishing it from standard multi-candidate protocols.})}}
  \label{alg:main_algorithm}
  \begin{algorithmic}[1]
    \STATE {\bfseries Input:} Prefix $x_{<n}$, target model $\mathcal{P}$, draft model $\mathcal{Q}$, and candidates $k$.
    \STATE {\bfseries Output:} A verified token $x$ sampled via the ResiSpec protocol.
    \STATE {\bfseries Initialize} residual $R \leftarrow \mathcal{P}$, draft $D \leftarrow \mathcal{Q}$
    \FOR{$i=1 \to k$}
        \STATE Sample $x_i \sim D, r_i \sim \text{Uniform}(0,1)$
        \IF{$r_i < \frac{R(x_i)}{D(x_i)}$}
            \STATE {\bfseries Return} $x_i$ \mycomment{Accept $x_i$}
        \ELSE
            \LINECOMMENT{Construct Proxy Distribution}
            \STATE $\underline{K \leftarrow \text{GetProxy}(R, Q, x_i)}$ \mycomment{\textit{ResiSpec proxy step}}
            \LINECOMMENT{Use Proxy Distribution to Calculate Residual Distribution instead of Draft distribution}
             \STATE $R \leftarrow \text{norm}(\max(R-K, 0))$ 
        \ENDIF
    \ENDFOR
    \LINECOMMENT{Resampling From the Residual Distribution}
    \STATE {\bfseries Return} $x \sim R$ 
    \item[] \hrulefill
    \FUNCTION{$\textsc{GetProxy}(p, q, x_0)$}
        \STATE $Z \gets q(x_0) - p(x_0)$ \mycomment{Total residual budget}
        \STATE $r_i \gets \min(Z \cdot q_i, p_i)$ for all $i \neq x_0$, and $r_{x_0} \gets 0$ 
        \STATE $\Delta \gets Z - \sum_i r_i$ \mycomment{Mass lost due to clipping at $p_i$}
        \STATE Update $r$ by distributing $\Delta$ into the remaining capacity $(p-r)$ 
        \mycomment{Ensures $\sum r = Z$ and $r \le p$}
        \STATE $k \gets p - r + s$, then set $k(x_0) \gets q(x_0)$ \mycomment{Align with $q$ at $x_0$}
        \STATE {\bfseries Return} $k$ 
    \ENDFUNCTION
  \end{algorithmic}
\end{algorithm}


Guided by the monotonicity principle (Theorem~\ref{thm:residual_shaping_monotonicity})—which suggests that any $Z$ larger than necessary will further deviate the residual from the draft's manifold—we select the minimal $Z$ that accommodates this local gap. In practice, we employ the following heuristic:
\begin{equation}
    Z_{applied} = \max(Z^*, |q(x_0) - p(x_0)|)
\end{equation}
This adjustment ensures that $k(x)$ remains a valid probability distribution while preventing the residual mass from collapsing due to long-tail noise. The complete procedure for dynamically adjusting $r(x)$ and $Z$ to balance theoretical exactness and inference speed is detailed in Algorithm~\ref{alg:resispec}. 

In the event of a verification failure, $\textsc{GetProxy}$ intervenes to adjust the residual distribution $R$ prior to the next verification step. Its primary role is to prevent the target distribution from drifting into regions of low draft density. By aligning the mass of $R$ with the draft model $Q$, the system ensures that the remaining candidates $\{x_{i+1}, \dots, x_k\}$ encounter verification criteria compatible with the draft model’s knowledge. The function’s logic is as follows:

\textbf{Step 1: Budgeting and Initial Projection.} It first calculates the mass budget $Z = |q(x_0) - p(x_0)|$. It then performs an initial projection by setting $r_i = \min(Zq_i, p_i)$, effectively ``clipping'' the desired draft shape into the target model’s available capacity.

\textbf{Step 2: Capacity-Aware Redistribution.} Any mass deficit $\Delta$ caused by clipping is redistributed to tokens where $p_i > r_i$. This ``water-filling'' approach ensures that the total residual mass strictly meets the budget $Z$ without violating the physical upper bound $p(x)$.

\textbf{Step 3: Proxy Formulation.} Finally, it constructs $k = p - r$, with a local adjustment at $x_0$ to ensure $k(x_0) = q(x_0)$. This ensures that the acceptance ratio at the failed point remains mathematically consistent with the original draft.


\subsection{Theoretical Verification}
A prerequisite for any speculative sampling framework is the preservation of the target model's distribution, ensuring that acceleration is achieved without compromising output quality. We formally prove that ResiSpec satisfies this exactness property; specifically, our iterative residual-shaping process, mediated by the proxy distribution $k(x)$, is shown to be provably unbiased, maintaining a marginal distribution identical to the target model $p(x)$ across all verification steps. A rigorous theoretical justification is provided in Appendix~\ref{appendix:exactness}. This theoretical guarantee is further corroborated by empirical validation, where we demonstrate zero distributional shift compared to standard auto-regressive decoding.

\section{Experiment}
\subsection{Experiment Setup}
\textbf{Model and Baseline.}
We evaluate ResiSpec in comparison with SpecInfer~\cite{specinfer} using Llama-2-7B~\cite{llama2} as the target model and JackFram/Llama-68M (JF68M)~\cite{specinfer} as the draft model. To ensure evaluation coverage, we conduct experiments on three widely-used text corpora: CNN/DM~\cite{cnn}, OpenWebText~\cite{openwebtext}, and C4~\cite{c4}. All evaluations are performed under two temperature settings, $T=1.0$ and $T=0.6$, to assess performance across different sampling regimes.

\textbf{Metrics.}
We assess the efficiency of speculative decoding using two metrics. The first metric is \emph{Accepted Length} (Acc. Len), which measures the average number of tokens successfully accepted per speculative decoding step; higher values indicate more effective draft utilization. The second metric is \emph{Throughput} (Thru.), defined as the number of generated tokens per second, which directly reflects end-to-end inference efficiency. All experiments are conducted on a single NVIDIA RTX 3090 GPU and averaged over multiple runs.

\subsection{Main Results}

We perform an extensive evaluation comparing ResiSpec with SpecInfer, with detailed results presented in Table~\ref{tab:results}. Considering that both ResiSpec and SpecInfer employ a tree-based speculative decoding method, we use a 3-branch complete tree with a depth of 5 for all experiments. The findings demonstrate that ResiSpec achieves statistically significant improvements over the baseline method. Specifically, we observe consistent enhancements in accepted length (up to \textbf{1.86$\times$}) and throughput (up to \textbf{1.68$\times$}) across all tested configurations. Notably, these performance gains are sustained across diverse data domains and varying sampling temperatures, highlighting the robustness and generalizability of our proposed distribution shaping mechanism.

\textbf{Compatibility with Existing Methods.}
To examine whether residual shaping is tied to a specific baseline, we integrate ResiSpec with two different multi-candidate speculative decoding methods. EAGLE improves the drafter side by accelerating candidate generation, while Sequoia optimizes the verification tree used by the target model. As shown in Table~\ref{tab:framework_compatibility}, ResiSpec is evaluated as an add-on, reporting accepted-length and throughput improvements under matched decoding settings.

\textbf{Computational Overhead.}
We further profile the runtime of ResiSpec's proxy construction and residual update relative to a single target-model verification step. Most of this overhead comes from memory-bound GPU operations over token distributions, rather than heavy arithmetic; fused kernels could therefore further reduce ResiSpec's latency. Importantly, ResiSpec does not introduce an additional model forward pass, and its cost comes from distribution-level operations after verification. Figure~\ref{fig:overhead_breakdown} reports per-step latency broken into drafting, verification, and ResiSpec shaping, and further decomposes ResiSpec's time into budgeting, redistribution, and proxy formulation.

\begin{figure}[t]
\centering
\includegraphics[width=\linewidth]{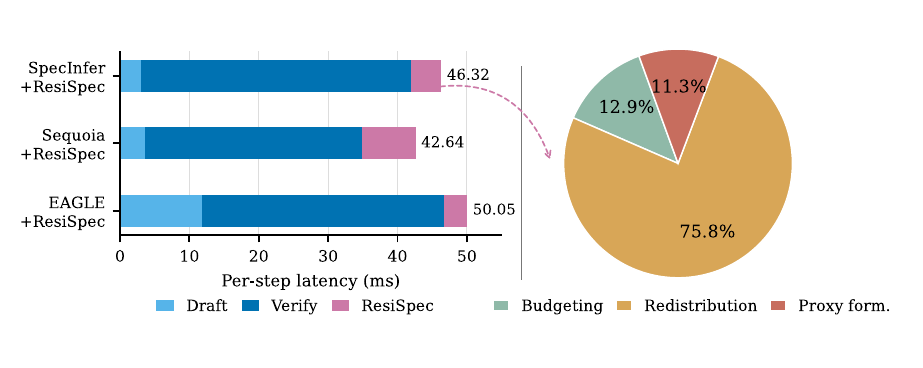}
\caption{Runtime analysis of ResiSpec. Left: per-step latency across SpecInfer, Sequoia, and EAGLE when augmented with ResiSpec, decomposed into drafting, target verification, and ResiSpec shaping. Right: ResiSpec's internal runtime breakdown across budgeting, redistribution, and proxy formulation, shown by percentage.}
\label{fig:overhead_breakdown}
\end{figure}

\subsection{Ablation Study}

To further dissect the efficacy of ResiSpec, we analyze both tree topology and target-draft model pairing. For topology, we compare tree structures under an iso-compute constraint by selecting configurations with similar verification costs, ensuring that performance differences reflect distribution alignment rather than extra parallel computation. Table~\ref{tab:ablation_results_1} shows that ResiSpec consistently outperforms SpecInfer, with larger gains on deeper and narrower trees, indicating stronger mitigation of accumulated Residual Drift over longer speculation horizons. We also evaluate additional target-draft model pairs and datasets in Appendix~\ref{appendix:additional_ablation}, where ResiSpec consistently improves accepted length and throughput over SpecInfer, suggesting that residual shaping is not tied to a single model family or capacity gap.

\begin{table}[htbp]
\centering
\small
\begin{tabular}{l c c c c} 
\toprule
\multirow{2}{*}{Tree} &
\multirow{2}{*}{\makecell{Verification\\Costs}} &
\multicolumn{1}{c}{ResiSpec} &
\multicolumn{1}{c}{SpecInfer} &
\multirow{2}{*}{SpeedUp} \\ 
\cmidrule(lr){3-3} \cmidrule(lr){4-4}
& & Thru.↑ & Thru.↑ & \\
\midrule
d7b2 & 127 & 101.01 & 52.69 & 1.92$\times$ \\
d5b3 & 121 &94.25  & 56.92 & 1.66$\times$ \\
d4b5 & 156 &60.83  & 42.46 & 1.43$\times$ \\
d3b11 & 133 & 44.21  & 37.89 & 1.17$\times$ \\
\bottomrule
\end{tabular}
\caption{Impact of tree topologies under comparable verification budgets on Llama-2-7B with JF68M. Verification cost is the number of candidate positions verified concurrently. ResiSpec gains more over SpecInfer on deeper, narrower trees.}
\label{tab:ablation_results_1}
\end{table}

\subsection{Empirical Validation of Distributional Exactness}

To empirically verify the theoretical exactness proven in Theorem~\ref{thm:exactness} in Appendix~\ref{appendix:exactness}, we measure the Kullback-Leibler (KL) divergence~\cite{kldiv} between the original target distribution $p(x)$ and the effective marginal distribution produced by ResiSpec. The effective distribution is reconstructed by exhaustively marginalizing the probability mass across all possible verification trajectories within a candidate batch, including all sequential acceptance and rejection paths (see Appendix~\ref{appendix:calculation_empirical} for the full derivation). 

As shown in Table~\ref{tab:KL Divergence}, we compare ResiSpec against standard multi-candidate speculative decoding. Our results demonstrate that ResiSpec maintains the target distribution with high fidelity. Specifically, the KL divergence between the ResiSpec-generated distribution and the target model remains on the order of $10^{-10}$, a level of discrepancy that is indistinguishable from standard multi-candidate methods. This negligible error is consistent with numerical floating-point noise rather than any systematic algorithmic bias. These findings confirm that ResiSpec is stochastically equivalent to the target model, ensuring that our distribution shaping does not sacrifice generation quality for speed.

\begin{table}[htbp]
\centering
\small

\begin{tabular}{@{}llcc@{}}
\toprule
\textbf{Method} & \textbf{Data} & \textbf{KL ($T=1.0$)} & \textbf{KL ($T=0.6$)} \\
\midrule
\multirow{3}{*}{ResiSpec} 
    & CNN/DM    & 2.64e-10 & 4.38e-10 \\
    & OpenWebText & 7.15e-10 & 4.37e-10 \\
    & C4          & 1.02e-9  & 1.39e-10 \\
\midrule
\multirow{3}{*}{SpecInfer} 
    & CNN/DM    & 1.04e-9  & 3.55e-10 \\
    & OpenWebText & 1.63e-9  & 3.89e-10 \\
    & C4          & 9.10e-10 & 1.09e-10 \\
\bottomrule
\end{tabular}
\caption{KL divergence between the target distribution $p(x)$ and the sampling distributions of ResiSpec and SpecInfer on Llama-2-7B with JF68M. Near-zero values across datasets and temperatures confirm distributional preservation.}
\label{tab:KL Divergence}
\end{table}

\section{Conclusion}
We address the ``residual drift'' bottleneck in multi-candidate speculative decoding by proposing ResiSpec, a distribution-shaping framework. By aligning rejection residuals with the draft model through a proxy distribution, ResiSpec prevents the collapse of acceptance rates during sequential verification. As a verification-time add-on, ResiSpec can also complement existing candidate-generation and tree-construction methods. Our method maintains exact sampling while achieving up to 1.92$\times$ speedup, providing a robust and efficient solution for high-throughput LLM acceleration.


\nocite{spec_recyc}
\nocite{adaptive_l_sd}
\nocite{rest}
\nocite{control}
\nocite{old_speculative}
\nocite{tinyllama}
\nocite{rejection_sampling}
\nocite{prml}


\bibliography{aaai2027}

\begin{thebibliography}{41}
\providecommand{\natexlab}[1]{#1}

\bibitem[{Achiam et~al.(2023)Achiam, Adler, Agarwal, Ahmad, Akkaya, Aleman,
  Almeida, Altenschmidt, Altman, Anadkat et~al.}]{gpt4}
Achiam, J.; Adler, S.; Agarwal, S.; Ahmad, L.; Akkaya, I.; Aleman, F.~L.;
  Almeida, D.; Altenschmidt, J.; Altman, S.; Anadkat, S.; et~al. 2023.
\newblock Gpt-4 technical report.
\newblock \emph{arXiv preprint arXiv:2303.08774}.

\bibitem[{Ankner et~al.(2024)Ankner, Parthasarathy, Nrusimha, Rinard,
  Ragan-Kelley, and Brandon}]{hydra}
Ankner, Z.; Parthasarathy, R.; Nrusimha, A.; Rinard, C.; Ragan-Kelley, J.; and
  Brandon, W. 2024.
\newblock Hydra: Sequentially-dependent draft heads for medusa decoding.
\newblock \emph{arXiv preprint arXiv:2402.05109}.

\bibitem[{Baevski and Auli(2018)}]{long_tail}
Baevski, A.; and Auli, M. 2018.
\newblock Adaptive input representations for neural language modeling.
\newblock \emph{arXiv preprint arXiv:1809.10853}.

\bibitem[{Bapna, Arivazhagan, and Firat(2020)}]{control}
Bapna, A.; Arivazhagan, N.; and Firat, O. 2020.
\newblock Controlling computation versus quality for neural sequence models.
\newblock \emph{arXiv preprint arXiv:2002.07106}.

\bibitem[{Bishop and Nasrabadi(2006)}]{prml}
Bishop, C.~M.; and Nasrabadi, N.~M. 2006.
\newblock \emph{Pattern recognition and machine learning}.
\newblock Springer.

\bibitem[{Burton(2012)}]{old_speculative}
Burton, F.~W. 2012.
\newblock Speculative computation, parallelism, and functional programming.
\newblock \emph{IEEE Transactions on Computers}, 100(12): 1190--1193.

\bibitem[{Cai et~al.(2024)Cai, Li, Geng, Peng, Lee, Chen, and Dao}]{medusa}
Cai, T.; Li, Y.; Geng, Z.; Peng, H.; Lee, J.~D.; Chen, D.; and Dao, T. 2024.
\newblock Medusa: Simple llm inference acceleration framework with multiple
  decoding heads.
\newblock \emph{arXiv preprint arXiv:2401.10774}.

\bibitem[{Chen et~al.(2023)Chen, Borgeaud, Irving, Lespiau, Sifre, and
  Jumper}]{specd_2}
Chen, C.; Borgeaud, S.; Irving, G.; Lespiau, J.-B.; Sifre, L.; and Jumper, J.
  2023.
\newblock Accelerating large language model decoding with speculative sampling.
\newblock \emph{arXiv preprint arXiv:2302.01318}.

\bibitem[{Chen et~al.(2024)Chen, May, Svirschevski, Huang, Ryabinin, Jia, and
  Chen}]{sequoia}
Chen, Z.; May, A.; Svirschevski, R.; Huang, Y.; Ryabinin, M.; Jia, Z.; and
  Chen, B. 2024.
\newblock Sequoia: Scalable, robust, and hardware-aware speculative decoding.
\newblock \emph{arXiv preprint arXiv:2402.12374}.

\bibitem[{Dao et~al.(2022)Dao, Fu, Ermon, Rudra, and R{\'e}}]{flash_attention}
Dao, T.; Fu, D.; Ermon, S.; Rudra, A.; and R{\'e}, C. 2022.
\newblock Flashattention: Fast and memory-efficient exact attention with
  io-awareness.
\newblock \emph{Advances in neural information processing systems}, 35:
  16344--16359.

\bibitem[{Fu et~al.(2024)Fu, Bailis, Stoica, and Zhang}]{lookahead}
Fu, Y.; Bailis, P.; Stoica, I.; and Zhang, H. 2024.
\newblock Break the sequential dependency of llm inference using lookahead
  decoding.
\newblock \emph{arXiv preprint arXiv:2402.02057}.

\bibitem[{Gokaslan and Cohen(2019)}]{openwebtext}
Gokaslan, A.; and Cohen, V. 2019.
\newblock OpenWebText Corpus.

\bibitem[{He et~al.(2024)He, Zhong, Cai, Lee, and He}]{rest}
He, Z.; Zhong, Z.; Cai, T.; Lee, J.; and He, D. 2024.
\newblock Rest: Retrieval-based speculative decoding.
\newblock In \emph{Proceedings of the 2024 conference of the North American
  chapter of the association for computational linguistics: Human language
  technologies (volume 1: long papers)}, 1582--1595.

\bibitem[{Hu et~al.(2025{\natexlab{a}})Hu, Liu, Dong, Peng, McDanel, and
  Zhang}]{survey_2025}
Hu, Y.; Liu, Z.; Dong, Z.; Peng, T.; McDanel, B.; and Zhang, S.~Q.
  2025{\natexlab{a}}.
\newblock Speculative decoding and beyond: An in-depth survey of techniques.
\newblock \emph{arXiv preprint arXiv:2502.19732}.

\bibitem[{Hu et~al.(2025{\natexlab{b}})Hu, Zheng, Viswanathan, Chen, Rossi, Wu,
  Manocha, and Huang}]{optimal_2025}
Hu, Z.; Zheng, T.; Viswanathan, V.; Chen, Z.; Rossi, R.~A.; Wu, Y.; Manocha,
  D.; and Huang, H. 2025{\natexlab{b}}.
\newblock Towards optimal multi-draft speculative decoding.
\newblock \emph{arXiv preprint arXiv:2502.18779}.

\bibitem[{Huang, Guo, and Wang(2024)}]{adaptive_l_sd}
Huang, K.; Guo, X.; and Wang, M. 2024.
\newblock Specdec++: Boosting speculative decoding via adaptive candidate
  lengths.
\newblock \emph{arXiv preprint arXiv:2405.19715}.

\bibitem[{Kullback and Leibler(1951)}]{kldiv}
Kullback, S.; and Leibler, R.~A. 1951.
\newblock On information and sufficiency.
\newblock \emph{The annals of mathematical statistics}, 22(1): 79--86.

\bibitem[{Leviathan, Kalman, and Matias(2023)}]{specud}
Leviathan, Y.; Kalman, M.; and Matias, Y. 2023.
\newblock Fast inference from transformers via speculative decoding.
\newblock In \emph{Proceedings of the 40th International Conference on Machine
  Learning}, ICML'23. JMLR.org.

\bibitem[{Li et~al.(2024{\natexlab{a}})Li, Wei, Zhang, and Zhang}]{eagle_2}
Li, Y.; Wei, F.; Zhang, C.; and Zhang, H. 2024{\natexlab{a}}.
\newblock Eagle-2: Faster inference of language models with dynamic draft
  trees.
\newblock \emph{arXiv preprint arXiv:2406.16858}.

\bibitem[{Li et~al.(2024{\natexlab{b}})Li, Wei, Zhang, and Zhang}]{eagle}
Li, Y.; Wei, F.; Zhang, C.; and Zhang, H. 2024{\natexlab{b}}.
\newblock EAGLE: speculative sampling requires rethinking feature uncertainty.
\newblock In \emph{Proceedings of the 41st International Conference on Machine
  Learning}, ICML'24. JMLR.org.

\bibitem[{Luo et~al.(2025)Luo, Wang, Zhu, Zhang, Zhang, Yang, and
  Xu}]{spec_recyc}
Luo, X.; Wang, Y.; Zhu, Q.; Zhang, Z.; Zhang, X.; Yang, Q.; and Xu, D. 2025.
\newblock Turning trash into treasure: Accelerating inference of large language
  models with token recycling.
\newblock In \emph{Proceedings of the 63rd Annual Meeting of the Association
  for Computational Linguistics (Volume 1: Long Papers)}, 6816--6831.

\bibitem[{Mann et~al.(2020)Mann, Ryder, Subbiah, Kaplan, Dhariwal, Neelakantan,
  Shyam, Sastry, Askell, Agarwal et~al.}]{gpt3}
Mann, B.; Ryder, N.; Subbiah, M.; Kaplan, J.; Dhariwal, P.; Neelakantan, A.;
  Shyam, P.; Sastry, G.; Askell, A.; Agarwal, S.; et~al. 2020.
\newblock Language models are few-shot learners.
\newblock \emph{arXiv preprint arXiv:2005.14165}, 1(3): 3.

\bibitem[{Miao et~al.(2024)Miao, Oliaro, Zhang, Cheng, Wang, Zhang, Wong, Zhu,
  Yang, Shi, Shi, Chen, Arfeen, Abhyankar, and Jia}]{specinfer}
Miao, X.; Oliaro, G.; Zhang, Z.; Cheng, X.; Wang, Z.; Zhang, Z.; Wong, R.
  Y.~Y.; Zhu, A.; Yang, L.; Shi, X.; Shi, C.; Chen, Z.; Arfeen, D.; Abhyankar,
  R.; and Jia, Z. 2024.
\newblock SpecInfer: Accelerating Large Language Model Serving with Tree-based
  Speculative Inference and Verification.
\newblock In \emph{Proceedings of the 29th ACM International Conference on
  Architectural Support for Programming Languages and Operating Systems, Volume
  3}, ASPLOS ’24, 932–949. ACM.

\bibitem[{Narayanan et~al.(2021)Narayanan, Shoeybi, Casper, LeGresley, Patwary,
  Korthikanti, Vainbrand, Kashinkunti, Bernauer, Catanzaro, Phanishayee, and
  Zaharia}]{megatron_LM}
Narayanan, D.; Shoeybi, M.; Casper, J.; LeGresley, P.; Patwary, M.;
  Korthikanti, V.~A.; Vainbrand, D.; Kashinkunti, P.; Bernauer, J.; Catanzaro,
  B.; Phanishayee, A.; and Zaharia, M. 2021.
\newblock Efficient Large-Scale Language Model Training on GPU Clusters Using
  Megatron-LM.
\newblock arXiv:2104.04473.

\bibitem[{Peng et~al.(2023)Peng, Li, He, Galley, and Gao}]{vicuna}
Peng, B.; Li, C.; He, P.; Galley, M.; and Gao, J. 2023.
\newblock Instruction tuning with gpt-4.
\newblock \emph{arXiv preprint arXiv:2304.03277}.

\bibitem[{Pope et~al.(2023)Pope, Douglas, Chowdhery, Devlin, Bradbury, Heek,
  Xiao, Agrawal, and Dean}]{efficiency_motivation}
Pope, R.; Douglas, S.; Chowdhery, A.; Devlin, J.; Bradbury, J.; Heek, J.; Xiao,
  K.; Agrawal, S.; and Dean, J. 2023.
\newblock Efficiently scaling transformer inference.
\newblock \emph{Proceedings of machine learning and systems}, 5: 606--624.

\bibitem[{Pope et~al.(2022)Pope, Douglas, Chowdhery, Devlin, Bradbury,
  Levskaya, Heek, Xiao, Agrawal, and Dean}]{efficiently_transformer}
Pope, R.; Douglas, S.; Chowdhery, A.; Devlin, J.; Bradbury, J.; Levskaya, A.;
  Heek, J.; Xiao, K.; Agrawal, S.; and Dean, J. 2022.
\newblock Efficiently Scaling Transformer Inference.
\newblock arXiv:2211.05102.

\bibitem[{Radford et~al.(2019)Radford, Wu, Child, Luan, Amodei, Sutskever
  et~al.}]{gpt2}
Radford, A.; Wu, J.; Child, R.; Luan, D.; Amodei, D.; Sutskever, I.; et~al.
  2019.
\newblock Language models are unsupervised multitask learners.
\newblock \emph{OpenAI blog}, 1(8): 9.

\bibitem[{Raffel et~al.(2020)Raffel, Shazeer, Roberts, Lee, Narang, Matena,
  Zhou, Li, and Liu}]{c4}
Raffel, C.; Shazeer, N.; Roberts, A.; Lee, K.; Narang, S.; Matena, M.; Zhou,
  Y.; Li, W.; and Liu, P.~J. 2020.
\newblock Exploring the Limits of Transfer Learning with a Unified Text-to-Text
  Transformer.
\newblock \emph{Journal of Machine Learning Research}, 21(140): 1--67.

\bibitem[{Robert, Casella, and Casella(1999)}]{rejection_sampling}
Robert, C.~P.; Casella, G.; and Casella, G. 1999.
\newblock \emph{Monte Carlo statistical methods}, volume~2.
\newblock Springer.

\bibitem[{See, Liu, and Manning(2017)}]{cnn}
See, A.; Liu, P.~J.; and Manning, C.~D. 2017.
\newblock Get To The Point: Summarization with Pointer-Generator Networks.
\newblock In \emph{Proceedings of the 55th Annual Meeting of the Association
  for Computational Linguistics}, 1073--1083.

\bibitem[{Spector and Re(2023)}]{staged}
Spector, B.; and Re, C. 2023.
\newblock Accelerating llm inference with staged speculative decoding.
\newblock \emph{arXiv preprint arXiv:2308.04623}.

\bibitem[{Touvron et~al.(2023{\natexlab{a}})Touvron, Lavril, Izacard, Martinet,
  Lachaux, Lacroix, Rozi{\`e}re, Goyal, Hambro, Azhar et~al.}]{llama}
Touvron, H.; Lavril, T.; Izacard, G.; Martinet, X.; Lachaux, M.-A.; Lacroix,
  T.; Rozi{\`e}re, B.; Goyal, N.; Hambro, E.; Azhar, F.; et~al.
  2023{\natexlab{a}}.
\newblock LLaMA: open and efficient foundation language models. arXiv.
\newblock \emph{arXiv preprint arXiv:2302.13971}.

\bibitem[{Touvron et~al.(2023{\natexlab{b}})Touvron, Martin, Stone, Albert,
  Almahairi, Babaei, Bashlykov, Batra, Bhargava, Bhosale et~al.}]{llama2}
Touvron, H.; Martin, L.; Stone, K.; Albert, P.; Almahairi, A.; Babaei, Y.;
  Bashlykov, N.; Batra, S.; Bhargava, P.; Bhosale, S.; et~al.
  2023{\natexlab{b}}.
\newblock Llama 2: Open foundation and fine-tuned chat models.
\newblock \emph{arXiv preprint arXiv:2307.09288}.

\bibitem[{Vaswani et~al.(2017)Vaswani, Shazeer, Parmar, Uszkoreit, Jones,
  Gomez, Kaiser, and Polosukhin}]{all_you_need}
Vaswani, A.; Shazeer, N.; Parmar, N.; Uszkoreit, J.; Jones, L.; Gomez, A.~N.;
  Kaiser, {\L}.; and Polosukhin, I. 2017.
\newblock Attention is all you need.
\newblock \emph{Advances in neural information processing systems}, 30.

\bibitem[{Xia et~al.(2024)Xia, Yang, Dong, Wang, Li, Ge, Liu, Li, and
  Sui}]{specbench}
Xia, H.; Yang, Z.; Dong, Q.; Wang, P.; Li, Y.; Ge, T.; Liu, T.; Li, W.; and
  Sui, Z. 2024.
\newblock Unlocking efficiency in large language model inference: A
  comprehensive survey of speculative decoding.
\newblock \emph{arXiv preprint arXiv:2401.07851}.

\bibitem[{Xia et~al.(2023)Xia, Gao, Zeng, and Chen}]{sheared_llama}
Xia, M.; Gao, T.; Zeng, Z.; and Chen, D. 2023.
\newblock Sheared llama: Accelerating language model pre-training via
  structured pruning.
\newblock \emph{arXiv preprint arXiv:2310.06694}.

\bibitem[{Yin et~al.(2024)Yin, Chen, Huang, and Wang}]{2024_nips_theoretical}
Yin, M.; Chen, M.; Huang, K.; and Wang, M. 2024.
\newblock A theoretical perspective for speculative decoding algorithm.
\newblock \emph{Advances in Neural Information Processing Systems}, 37:
  128082--128117.

\bibitem[{Yu et~al.(2022)Yu, Jeong, Kim, Kim, and Chun}]{orca}
Yu, G.-I.; Jeong, J.~S.; Kim, G.-W.; Kim, S.; and Chun, B.-G. 2022.
\newblock Orca: A distributed serving system for $\{$Transformer-Based$\}$
  generative models.
\newblock In \emph{16th USENIX Symposium on Operating Systems Design and
  Implementation (OSDI 22)}, 521--538.

\bibitem[{Zhang, Zhao, and Chen(2024)}]{koala}
Zhang, K.; Zhao, J.; and Chen, R. 2024.
\newblock KOALA: Enhancing Speculative Decoding for LLM via Multi-Layer Draft
  Heads with Adversarial Learning.
\newblock \emph{arXiv preprint arXiv:2408.08146}.

\bibitem[{Zhang et~al.(2024)Zhang, Zeng, Wang, and Lu}]{tinyllama}
Zhang, P.; Zeng, G.; Wang, T.; and Lu, W. 2024.
\newblock Tinyllama: An open-source small language model.
\newblock \emph{arXiv preprint arXiv:2401.02385}.

\end{thebibliography}

\onecolumn
\appendix
\setcounter{secnumdepth}{2}

\section{Theoretical Proof for ResiSpec Exactness}

\subsection{Correctness and Exactness}
\label{appendix:exactness}

\begin{theorem}[Universal Exactness of ResiSpec]
\label{thm:exactness}
Let $p(x)$ be the target distribution and $q(x)$ be the draft distribution. For a candidate token $x \sim q(x)$, let $k(x)$ be a proxy distribution that satisfies the \textbf{Local Alignment} condition: $k(x) = q(x)$ at the specific value of the sampled candidate. If we define the acceptance probability as:
\begin{equation}
    \alpha(x) = \min\left(1, \frac{p(x)}{k(x)}\right)
\end{equation}
and the residual distribution as:
\begin{equation}
    p_{\text{res}}(x) = \frac{p(x) - \min(p(x), k(x))}{1 - \beta}
\end{equation}
where $\beta = \sum_{x \in \mathcal{X}} \min(p(x), k(x))$, then the marginal distribution of the output token $X$ is exactly $p(x)$.
\end{theorem}

\begin{proof}
The core intuition of ResiSpec is that while the proxy distribution $k(x)$ is used to shape the residual, the sampling exactness is preserved as long as $k(x)$ ``mimics'' the draft distribution $q(x)$ at the point of verification. To prove this, we consider the marginal probability $P(X=x')$ for any token $x' \in \mathcal{X}$ by summing the probabilities of two mutually exclusive events: (1) $x'$ is proposed and accepted, or (2) a candidate is proposed and rejected, and the system subsequently samples $x'$ from the residual distribution.

\textbf{Step 1: Analysis of the Acceptance Path.}
For the output $X$ to be $x'$ via acceptance, the draft model must first propose $x \sim q(x)$ where $x=x'$, and the verification step must subsequently accept it. Crucially, at this specific point of verification, the Local Alignment condition $k(x') = q(x')$ is active. The probability mass contributed by this path is:
\begin{equation}
\begin{aligned}
    P(\text{Accept}, X = x') &= q(x') \cdot \alpha(x') \\
    &= q(x') \cdot \min\left(1, \frac{p(x')}{k(x')}\right)
\end{aligned}
\end{equation}
By substituting the Local Alignment condition $k(x') = q(x')$, we can rewrite the acceptance probability in terms of $k$:
\begin{equation}
\label{eq:acc_path}
    P(\text{Accept}, X = x') = k(x') \cdot \min\left(1, \frac{p(x')}{k(x')}\right) = \min(k(x'), p(x'))
\end{equation}
This step is vital: it shows that by aligning $k$ with $q$ locally, the acceptance path effectively samples from the intersection of $p$ and our proxy $k$.

\textbf{Step 2: Analysis of the Rejection Path.}
The rejection path occurs if any proposed candidate $x$ is rejected. The total probability of rejection, $P(\text{Reject})$, is the complement of the total acceptance mass:
\begin{equation}
    P(\text{Reject}) = 1 - \sum_{x \in \mathcal{X}} P(\text{Accept}, X = x) = 1 - \sum_{x \in \mathcal{X}} \min(p(x), k(x)) = 1 - \beta
\end{equation}
Once a rejection occurs, the system samples $x'$ from the residual distribution $p_{\text{res}}(x')$. The probability mass contributed by this path is the product of the total rejection probability and the conditional probability of sampling $x'$:
\begin{equation}
\label{eq:rej_path}
\begin{aligned}
    P(\text{Reject}, X = x') &= P(\text{Reject}) \cdot p_{\text{res}}(x') \\
    &= (1 - \beta) \cdot \frac{p(x') - \min(p(x'), k(x'))}{1 - \beta} \\
    &= p(x') - \min(p(x'), k(x'))
\end{aligned}
\end{equation}
Note that the normalization constant $1-\beta$ cancels out, leaving a residual mass that perfectly complements the acceptance mass defined in Eq.~\ref{eq:acc_path}.

\textbf{Step 3: Summation of Marginal Probabilities.}
According to the law of total probability, the final marginal distribution $P(X=x')$ is the sum of the results from Eq.~\ref{eq:acc_path} and Eq.~\ref{eq:rej_path}:
\begin{equation}
\begin{aligned}
    P(X = x') &= P(\text{Accept}, X = x') + P(\text{Reject}, X = x') \\
    &= \min(p(x'), k(x')) + \left[ p(x') - \min(p(x'), k(x')) \right] \\
    &= p(x')
\end{aligned}
\end{equation}
Since this identity holds for all $x' \in \mathcal{X}$, the output distribution is identically the target distribution $p(x)$. 

\textbf{Conclusion.} 
The proof demonstrates that ResiSpec remains an exact sampling algorithm. By coupling the residual distribution to the same proxy $k(x)$ used in the acceptance test, any deviation of $k(x)$ from the original draft distribution $q(x)$ is mathematically compensated for in the failure state, provided that $k$ and $q$ are aligned at the point of verification.
\end{proof}

\subsection{Monotonicity of Residual Shaping Loss with Redistribution}
\label{appendix:Monotonicity}
\begin{theorem}[Monotonicity with Redistribution]
\label{thm:residual_shaping_monotonicity}
Let $q(x)$ and $p(x)$ be the draft and target distributions, respectively. Define the unnormalized residual mass $r(x)$ subject to the capacity constraint $0 \leq r(x) \leq p(x)$ and the global mass constraint $\sum_x r(x) = Z$. Let $p'(x) = r(x)/Z$ be the normalized residual distribution. 
The minimum $L_1$ distance $\mathcal{D}_{L1}(Z) = \sum_x |p'(x) - q(x)|$ is:
\begin{enumerate}
    \item Identically zero for $Z \in (0, Z^*]$, where $Z^* = \min_{x: q(x)>0} \frac{p(x)}{q(x)}$.
    \item Strictly monotonically increasing for $Z \in (Z^*, 1]$.
\end{enumerate}
\end{theorem}

\begin{proof}
To minimize $\mathcal{D}_{L1}(Z) = \frac{1}{Z} \sum_x |r(x) - Zq(x)|$, we define the local deviation as $E(x) = r(x) - Zq(x)$. Since $\sum_x r(x) = Z$ and $\sum_x Zq(x) = Z$, the total deviation sums to zero: $\sum_x E(x) = 0$. This implies that the sum of positive deviations must exactly balance the sum of negative deviations:
\begin{equation}
    \sum_{x: E(x)>0} |E(x)| = \sum_{x: E(x)<0} |E(x)|
\end{equation}
Consequently, the $L_1$ distance can be expressed solely in terms of the negative deviations:
\begin{equation}
    \mathcal{D}_{L1}(Z) = \frac{1}{Z} \left( \sum_{E(x)>0} |E(x)| + \sum_{E(x)<0} |E(x)| \right) = \frac{2}{Z} \sum_{x: E(x)<0} |E(x)|
\end{equation}

Under the capacity constraint $r(x) \leq p(x)$, a negative deviation ($E(x) < 0$) is mandatory if and only if $p(x) < Zq(x)$. To minimize the total deviation, we must set $r(x) = p(x)$ in these ``bottleneck'' regions and $r(x) \geq Zq(x)$ elsewhere (Redistribution). Thus, the minimal negative deviation at any point $x$ is $|E(x)| = \max(0, Zq(x) - p(x))$. Substituting this into the $L_1$ objective:
\begin{equation}
    \mathcal{D}_{L1}(Z) = \frac{2}{Z} \sum_{x} \max(0, Zq(x) - p(x)) = 2 \sum_x \max\left(0, q(x) - \frac{p(x)}{Z}\right)
\end{equation}

We now analyze the behavior of $\mathcal{D}_{L1}(Z)$ across the two regimes of $Z$:

\textbf{Case 1: $Z \in (0, Z^*]$.} By the definition of $Z^* = \min \frac{p(x)}{q(x)}$, for any $Z \leq Z^*$, we have $\frac{p(x)}{Z} \geq \frac{p(x)}{Z^*} \geq q(x)$ for all $x$. Thus, $q(x) - \frac{p(x)}{Z} \leq 0$ for all $x$, and $\mathcal{D}_{L1}(Z) = 0$. In this regime, the target model has sufficient capacity to perfectly mirror the draft distribution's shape.

\textbf{Case 2: $Z \in (Z^*, 1]$.} In this regime, the set of bottleneck tokens $\mathcal{S}_Z = \{x \mid q(x) > \frac{p(x)}{Z}\}$ is non-empty. Let $f(Z) = 2 \sum_{x \in \mathcal{S}_Z} (q(x) - \frac{p(x)}{Z})$. Taking the derivative with respect to $Z$:
\begin{equation}
    \frac{d \mathcal{D}_{L1}}{dZ} = 2 \sum_{x \in \mathcal{S}_Z} \frac{p(x)}{Z^2}
\end{equation}
Since $p(x) > 0$ for $x \in \mathcal{S}_Z$ and $Z^2 > 0$, we have $\frac{d \mathcal{D}_{L1}}{dZ} > 0$. This proves that the $L_1$ distance is strictly monotonically increasing with $Z$ once the residual mass exceeds the critical bottleneck capacity $Z^*$. 

\textbf{Conclusion:} Any mass $Z$ allocated beyond $Z^*$ necessitates a ``clipping'' of the draft shape at bottleneck points and a corresponding ``redistribution'' of mass to non-bottleneck points. Both actions contribute equally to the increase in $L_1$ distance, justifying the use of a minimal mass budget to preserve distribution fidelity.
\end{proof}

\section{Calculation of Empirical Validation of Distributional Exactness}
\label{appendix:calculation_empirical}

To empirically verify that ResiSpec maintains zero-bias sampling in a multi-candidate setting (e.g., batch or tree-based speculation), we measure the discrepancy between the theoretical target distribution $p(x)$ and the aggregate marginal distribution produced by the ResiSpec verification chain. In a multi-candidate scenario with $n$ candidates $\{x_1, x_2, \dots, x_n\}$, the final output token $X$ is determined by a sequential search for the first accepted candidate, or a fallback to a residual distribution if all candidates are rejected.

\subsection{Path-based Marginal Probability Calculation}
The empirical marginal distribution $P(X=x')$ is calculated by summing the probability masses of all mutually exclusive execution paths. Let $\alpha_i$ denote the acceptance probability of the $i$-th candidate and $1-\beta_i$ denote the probability of rejection at step $i$. The possible paths are:

\begin{enumerate}
    \item \textbf{Path 1: Accepted at the first candidate.} The probability that the first candidate $x_1$ is $x'$ and is accepted is:
    \begin{equation}
        P_1(x') = \min(k_1(x'), p(x'))
    \end{equation}
    
    \item \textbf{Path $i$ ($2 \leq i \leq n$): Accepted at the $i$-th candidate.} For the system to accept the $i$-th candidate, all preceding $i-1$ candidates must have been rejected. Due to the exactness property of the ResiSpec residual (see Appendix~\ref{appendix:exactness}), the target distribution effectively shifts to $p_{\text{res}, i-1}(x)$ after $i-1$ rejections. The probability mass is:
    \begin{equation}
        P_i(x') = \left( \prod_{j=1}^{i-1} (1 - \beta_j) \right) \cdot \min(k_i(x'), p_{\text{res}, i-1}(x'))
    \end{equation}
    where $p_{\text{res}, i-1}$ is the residual distribution reshaped by ResiSpec after the $(i-1)$-th rejection.

    \item \textbf{Path $n+1$: All candidates rejected.} If all $n$ candidates fail verification, the system samples $x'$ directly from the final residual distribution $p_{\text{res}, n}(x')$. The probability mass is:
    \begin{equation}
        P_{\text{fail}}(x') = \left( \prod_{j=1}^{n} (1 - \beta_j) \right) \cdot p_{\text{res}, n}(x')
    \end{equation}
\end{enumerate}

\subsection{Verification via KL Divergence}
According to Theorem~\ref{thm:exactness}, the total marginal probability $\hat{p}(x') = \sum_{i=1}^n P_i(x') + P_{\text{fail}}(x')$ must identically equal the target distribution $p(x')$. We evaluate this identity by calculating the \textbf{Kullback-Leibler (KL) Divergence} between the reconstructed distribution $\hat{p}$ and the ground-truth target distribution $p$:
\begin{equation}
    D_{KL}(\hat{p} \parallel p) = \sum_{x \in \mathcal{X}} \hat{p}(x) \log \frac{\hat{p}(x)}{p(x)}
\end{equation}
In our implementation, we recorded the full logits for both the draft and target models across various benchmarks. Across all test cases, the measured $D_{KL}$ consistently remained near zero (typically $< 10^{-8}$), which is within the expected range of floating-point numerical error. This empirical result confirms that ResiSpec introduces no distributional bias, ensuring that the accelerated inference remains mathematically equivalent to standard auto-regressive decoding.

\section{Reproducibility Details}
\label{appendix:reproducibility}

\noindent\textbf{Execution environment.}
All throughput experiments reported in the revised evaluation are run inside the same Docker environment to avoid differences caused by host-side Python, CUDA, or library versions. For each matched comparison, the baseline and the corresponding ResiSpec variant use the same container image, the same model checkpoints, and the same dataset order. We bind each run to a single GPU with \texttt{CUDA\_VISIBLE\_DEVICES=0}. Baseline and ResiSpec runs are executed sequentially, with one independent Python process per dataset-method pair, so that the GPU memory allocator state and kernel scheduling of one method do not interfere with the other. We record raw logs and parsed CSV summaries for every run, and compute throughput from the number of generated tokens divided by measured wall-clock decoding time after warmup.

\noindent\textbf{Models and datasets.}
The main fixed-tree experiments use Llama-2-7B as the target model and JackFram/Llama-68M as the draft model. To address experimental-scope concerns, we additionally evaluate Vicuna-7B-v1.3 with Llama-68M, Sheared-LLaMA-2.7B with TinyLlama-1.1B, and Sheared-LLaMA-2.7B with Sheared-LLaMA-1.3B under the same fixed-tree SpecInfer-style setting. The compatibility experiments use Llama-2-7B with Llama-68M for Sequoia, and Llama-2-7B-Chat with the corresponding EAGLE Llama-2-Chat drafter for EAGLE. Unless otherwise stated, all model checkpoints are loaded from the same local HuggingFace cache used by the Docker container. The dataset suite is CNN/DM, OpenWebText, and C4. For the 200-example runs, we use a fixed dataset order and evaluate the same prompt indices for the baseline and ResiSpec under each configuration. EAGLE prompt-generation runs use two warmup prompts followed by 200 measured prompts.

\noindent\textbf{Decoding configurations.}
For the SpecInfer-style fixed-tree experiments, both the baseline and ResiSpec use the same complete 3-branch tree with depth 5 and a maximum sequence length budget of 512. We evaluate both $T=0.6$ and $T=1.0$ with nucleus probability $p=1.0$. For dynamic-tree Sequoia, we use the original Sequoia tree-selection logic with the same L40 grow maps for the baseline and ResiSpec, and set $M=384$, $T=1.0$, and $p=1.0$ for the 200-example comparison. For EAGLE, we use \texttt{total\_token=60}, \texttt{depth=5}, and \texttt{max\_new\_tokens=128}; the main compatibility table uses \texttt{top\_k=10}, and the sensitivity rerun with \texttt{top\_k=15} follows the same Docker, dataset, and sequential-execution protocol.

\noindent\textbf{ResiSpec integration.}
In fixed-tree SpecInfer-style decoding, ResiSpec reshapes the residual distribution at each sequential verification position in the candidate tree. In Sequoia, ResiSpec is added after Sequoia has selected its dynamic tree, so the tree budget and topology are controlled by the original Sequoia policy and only the verification-time sampling distribution is changed. In EAGLE, ResiSpec is applied only to consecutive candidate siblings that share the same parent node. After each rejection, the implementation updates the current target distribution $p$ and draft/proxy distribution $q$ exactly as in the ResiSpec recurrence before verifying the next sibling.

\noindent\textbf{Randomness control.}
All sampling-based experiments use fixed pseudo-random seeds. Before each benchmark run, we initialize Python's \texttt{random} module, NumPy, PyTorch CPU random state, and PyTorch CUDA random states with the same base seed, and set cuDNN to deterministic mode. To make per-example sampling reproducible while avoiding identical random streams across prompts, the $i$-th evaluated prompt uses seed $s+i$, where $s$ is the base seed for that run. The same seed schedule is applied to ResiSpec and the corresponding baseline under each matched model, dataset, temperature, and tree configuration.

\noindent\textbf{Timing and diagnostics.}
Throughput runs disable auxiliary KL diagnostics and other distribution-checking instrumentation, because these checks require additional logit materialization and are not part of the decoding algorithm. Distribution preservation is evaluated separately using the procedure in Appendix~\ref{appendix:calculation_empirical}. For the runtime breakdown in Figure~\ref{fig:overhead_breakdown}, we add CUDA synchronization around the measured regions and report per-speculation-step latency for draft generation, target-model verification, and ResiSpec proxy computation. These profiling runs are used only for component attribution; the synchronized timings are not mixed with the end-to-end throughput measurements.

\section{Ablation Experiment Results}
\label{appendix:additional_ablation}

\begin{table}[htbp]
\centering
\small
\begin{tabular}{@{}ccccc@{}}
\toprule
\textbf{Candidates ($N$)} & \textbf{EAL↑}& \makecell[c]{\textbf{Marginal Gain}\\\textbf{($\Delta$EAL)}} & \textbf{Efficiency ($\eta$)} \\ \midrule
1 & 1.89 &  - & - \\
2 & 2.17 &  0.28 & 0.280 \\
4 & 2.48 &  0.31 & 0.155 \\
8 & 2.76 &  0.28 & 0.070 \\
\bottomrule
\end{tabular}
\caption{
\textbf{Scaling Inefficiency of Multi-Candidate Schemes.} With speculation depth fixed at 4, increasing candidates $N$ yields only modest EAL gains while the efficiency score $\eta$ drops sharply. Results are averaged over CNN/DM using Llama-2-7B.
}
\label{tab:scaling_inefficiency}
\end{table}

\begin{table*}[htbp]
  \centering
  \small
    \begin{tabular}{@{}llclcrrrr@{}}
      \toprule
      Target & Draft & T & Data & Tree & Resi. Acc. & Resi. Thru. & Spec. Acc. & Spec. Thru. \\
      \midrule
    Sheared-2.7B & Sheared-1.3B & 1 & CNN/DM & d5b3 & 4.88 & 78.309 & 3.88 & 66.313 \\
    Sheared-2.7B & Sheared-1.3B & 1 & CNN/DM & d4b4 & 3.94 & 82.988 & 3.42 & 76.695 \\
    Sheared-2.7B & Sheared-1.3B & 1 & OpenWebText & d5b3 & 4.86 & 78.447 & 3.92 & 67.660 \\
    Sheared-2.7B & Sheared-1.3B & 1 & OpenWebText & d4b4 & 3.95 & 83.060 & 3.45 & 76.864 \\
    Sheared-2.7B & Sheared-1.3B & 1 & C4 & d5b3 & 4.84 & 78.362 & 3.98 & 68.679 \\
    Sheared-2.7B & Sheared-1.3B & 1 & C4 & d4b4 & 3.95 & 82.850 & 3.44 & 77.340 \\
    Sheared-2.7B & Sheared-1.3B & 0.6 & CNN/DM & d5b3 & 4.88 & 78.927 & 4.12 & 71.174 \\
    Sheared-2.7B & Sheared-1.3B & 0.6 & CNN/DM & d4b4 & 3.96 & 83.110 & 3.55 & 79.681 \\
    Sheared-2.7B & Sheared-1.3B & 0.6 & OpenWebText & d5b3 & 4.89 & 78.998 & 4.10 & 70.922 \\
    Sheared-2.7B & Sheared-1.3B & 0.6 & OpenWebText & d4b4 & 3.97 & 81.833 & 3.54 & 78.424 \\
    Sheared-2.7B & Sheared-1.3B & 0.6 & C4 & d5b3 & 4.88 & 77.939 & 4.13 & 70.685 \\
    Sheared-2.7B & Sheared-1.3B & 0.6 & C4 & d4b4 & 3.97 & 82.559 & 3.51 & 78.083 \\
    Sheared-2.7B & TinyLlama-1.1B & 1 & CNN/DM & d5b3 & 4.79 & 82.170 & 3.71 & 68.587 \\
    Sheared-2.7B & TinyLlama-1.1B & 1 & CNN/DM & d4b4 & 3.92 & 86.656 & 3.30 & 78.432 \\
    Sheared-2.7B & TinyLlama-1.1B & 1 & OpenWebText & d5b3 & 4.83 & 82.798 & 3.69 & 68.074 \\
    Sheared-2.7B & TinyLlama-1.1B & 1 & OpenWebText & d4b4 & 3.92 & 86.496 & 3.30 & 78.125 \\
    Sheared-2.7B & TinyLlama-1.1B & 1 & C4 & d5b3 & 4.81 & 82.645 & 3.71 & 68.493 \\
    Sheared-2.7B & TinyLlama-1.1B & 1 & C4 & d4b4 & 3.93 & 86.656 & 3.27 & 77.580 \\
    Sheared-2.7B & TinyLlama-1.1B & 0.6 & CNN/DM & d5b3 & 4.84 & 82.508 & 3.87 & 71.609 \\
    Sheared-2.7B & TinyLlama-1.1B & 0.6 & CNN/DM & d4b4 & 3.95 & 86.957 & 3.43 & 81.699 \\
    Sheared-2.7B & TinyLlama-1.1B & 0.6 & OpenWebText & d5b3 & 4.79 & 82.441 & 3.80 & 70.423 \\
    Sheared-2.7B & TinyLlama-1.1B & 0.6 & OpenWebText & d4b4 & 3.95 & 87.395 & 3.32 & 79.005 \\
    Sheared-2.7B & TinyLlama-1.1B & 0.6 & C4 & d5b3 & 4.81 & 82.508 & 3.79 & 70.077 \\
    Sheared-2.7B & TinyLlama-1.1B & 0.6 & C4 & d4b4 & 3.93 & 87.036 & 3.32 & 78.858 \\
    Sheared-2.7B & Tiny-Vicuna-1B & 1 & CNN/DM & d5b3 & 4.81 & 82.508 & 3.53 & 65.274 \\
    Sheared-2.7B & Tiny-Vicuna-1B & 1 & CNN/DM & d4b4 & 3.93 & 86.496 & 3.18 & 75.429 \\
    Sheared-2.7B & Tiny-Vicuna-1B & 1 & OpenWebText & d5b3 & 4.81 & 82.170 & 3.51 & 65.050 \\
    Sheared-2.7B & Tiny-Vicuna-1B & 1 & OpenWebText & d4b4 & 3.93 & 86.656 & 3.17 & 75.302 \\
    Sheared-2.7B & Tiny-Vicuna-1B & 1 & C4 & d5b3 & 4.81 & 82.441 & 3.58 & 66.327 \\
    Sheared-2.7B & Tiny-Vicuna-1B & 1 & C4 & d4b4 & 3.93 & 86.957 & 3.21 & 76.278 \\
    Sheared-2.7B & Tiny-Vicuna-1B & 0.6 & CNN/DM & d5b3 & 4.78 & 81.967 & 3.71 & 68.823 \\
    Sheared-2.7B & Tiny-Vicuna-1B & 0.6 & CNN/DM & d4b4 & 3.95 & 87.103 & 3.25 & 77.519 \\
    Sheared-2.7B & Tiny-Vicuna-1B & 0.6 & OpenWebText & d5b3 & 4.79 & 82.051 & 3.64 & 67.510 \\
    Sheared-2.7B & Tiny-Vicuna-1B & 0.6 & OpenWebText & d4b4 & 3.95 & 86.957 & 3.21 & 76.588 \\
    Sheared-2.7B & Tiny-Vicuna-1B & 0.6 & C4 & d5b3 & 4.79 & 81.882 & 3.65 & 67.660 \\
    Sheared-2.7B & Tiny-Vicuna-1B & 0.6 & C4 & d4b4 & 3.95 & 87.191 & 3.24 & 77.058 \\
    Llama-2-7B & JF68M & 1 & CNN/DM & d5b3 & 4.68 & 95.057 & 2.51 & 56.851 \\
    Llama-2-7B & JF68M & 1 & CNN/DM & d4b4 & 3.88 & 84.317 & 2.47 & 58.962 \\
    Llama-2-7B & JF68M & 1 & OpenWebText & d5b3 & 4.69 & 94.004 & 2.51 & 56.980 \\
    Llama-2-7B & JF68M & 1 & OpenWebText & d4b4 & 3.87 & 84.246 & 2.47 & 59.172 \\
    Llama-2-7B & JF68M & 1 & C4 & d5b3 & 4.67 & 94.429 & 2.64 & 59.952 \\
    Llama-2-7B & JF68M & 1 & C4 & d4b4 & 3.89 & 84.674 & 2.55 & 60.976 \\
    Llama-2-7B & JF68M & 0.6 & CNN/DM & d5b3 & 4.14 & 84.603 & 2.59 & 58.997 \\
    Llama-2-7B & JF68M & 0.6 & CNN/DM & d4b4 & 3.77 & 82.035 & 2.47 & 59.172 \\
    Llama-2-7B & JF68M & 0.6 & OpenWebText & d5b3 & 4.16 & 84.531 & 2.49 & 56.690 \\
    Llama-2-7B & JF68M & 0.6 & OpenWebText & d4b4 & 3.75 & 81.566 & 2.42 & 58.072 \\
    Llama-2-7B & JF68M & 0.6 & C4 & d5b3 & 4.13 & 84.388 & 2.66 & 60.570 \\
    Llama-2-7B & JF68M & 0.6 & C4 & d4b4 & 3.74 & 81.499 & 2.57 & 61.652 \\

      \bottomrule
      \multicolumn{9}{l}{\small \textit{Notes:}} \\
      \multicolumn{9}{l}{\small $T$: temperature; Tree: prediction tree shape; Acc.: accepted length; Thru.: tokens per second.} 
    \end{tabular}%
  \caption{Ablation Experiment Results of Different Model Pair Combinations. Here, we use Llama-2-7B and Sheared-LLaMA-2.7B~\cite{sheared_llama} as the target model and Sheared-LLaMA-1.3B, TinyLlama-1.1B~\cite{tinyllama}, Tiny-Vicuna-1B and JF68M as the draft model. The experimental results show that ResiSpec significantly outperforms SpecInfer in terms of accepted length and throughput under different configurations.
  }
  \label{tab:ablation_results_2}
  \label{tab:ablation_result}
  
\end{table*}


\end{document}